\documentclass{article}
\usepackage{float}
\usepackage{amsmath}
\usepackage{amsthm}
\usepackage{amssymb}
\usepackage{graphicx} 
\usepackage{subfigure}

\usepackage{natbib}

\usepackage{algorithm}
\usepackage{algorithmic}

\usepackage{hyperref}

\usepackage{multirow}

 \usepackage[accepted]{icml2013}

\newtheorem{definition}{Definition}

\newtheorem{lemma}{Lemma}

\newtheorem{theorem}{Theorem}

\floatstyle{ruled}
\newfloat{algorithm}{tbp}{loa}
\providecommand{\algorithmname}{Algorithm}
\floatname{algorithm}{\protect\algorithmname}

\icmltitlerunning{Loss-Proportional Subsampling for Subsequent ERM}

\begin{document}

\twocolumn[
\icmltitle{Loss-Proportional Subsampling for Subsequent ERM}

\icmlauthor{Paul Mineiro}{pmineiro@microsoft.com}
\icmladdress{Microsoft Cloud and Information Services Laboratory,
            One Microsoft Way, Redmond, WA 98052, USA}
\icmlauthor{Nikos Karampatziakis}{nikosk@microsoft.com}
\icmladdress{Microsoft Cloud and Information Services Laboratory,
            One Microsoft Way, Redmond, WA 98052, USA}

\icmlkeywords{subsampling, large-scale learning, empirical Bernstein}

\vskip 0.3in
]

\begin{abstract}
We propose a sampling scheme suitable for reducing a data set prior to
selecting a hypothesis with minimum empirical risk.  The sampling only
considers a subset of the ultimate (unknown) hypothesis set, but can
nonetheless guarantee that the final excess risk will compare favorably
with utilizing the entire original data set. We demonstrate the practical
benefits of our approach on a large dataset which we subsample and
subsequently fit with boosted trees.
%
\end{abstract}

\section{Introduction}

Data volumes are growing at a faster rate than available computing power,
storage space, or network bandwidth.  This has fueled interest in
distributed approaches to machine learning. However, the substantial jump
in communication costs between a multicore and a multicomputer system
currently confines many popular techniques to the single computer regime.
Consequently, even those machine learning workflows that today originate
with distributed data in a clustered environment often terminate with
the learning problem being solved on a single machine.  Furthermore,
because of computation, storage, or network limitations there is often
a subsampling step between the data store and the single machine.

Our concern here is to make the subsampling step as statistically
efficient as possible, knowing that the data will subsequently be given
to an empirical risk minimization (ERM) algorithm.  What properties can
we hope to have from a subsample of the original data?  If we
were to perform ERM on the original data, the excess risk would be
$O (1 / \sqrt{n})$, where $n$ is the size of the original data set.
If we uniformly subsampled the original data set to size $m$ and 
performed ERM on this subsample, the excess risk would be $O
(1 / \sqrt{m})$.  From the perspective of ERM this compression is lossy.
Ideally, we would like to reduce the data set size while still retaining
excess risk $O (1 / \sqrt{n})$.

At first blush, this appears to be the active learning scenario. However,
the subsampling strategy has access to all examples and labels but cannot assume
knowledge of the hypothesis set ultimately used for ERM.  This is
because the types of hypotheses ultimately considered are presumed to
be intractable on the full dataset.  Instead we will only assume
access to a subset of the (otherwise unknown) hypothesis set.  As an example, 
suppose that the training on the subsample will be via neural networks with an
unknown architecture, but known to have direct connections from the input
layer to the output layer.  In this case, the set of linear predictors is
a subset of the final hypothesis space, and linear learning is feasible
at terafeature scale~\cite{Agarwal11}. Our results show we can compress
the data set by encoding relative to the mistakes of a linear predictor,
without distorting the subsequent ERM over neural networks; the amount
of compression possible is limited by the quality of the linear predictor.

In effect many current workflows look like 1) a subsampling step followed
by 2) a model selection step.  We propose to replace the first step by
1a) a simpler model selection step, followed by 1b) a subsampling step.
Although our approach does apply recursively, in practice we believe
much of the benefit would be captured by the introduction of a single
simple model selection step prior to subsampling.

\subsection{Relation to Prior Work}

This work bears strong resemblance to multiple threads of research,
and the main contribution is interpreting previous understanding 
and practices through the lens of empirical Bernstein bounds~\cite{MaurerP09}.

Boosting algorithms have popularized the idea of sequential model
selection via importance weighting.  Of particular relevance is
FilterBoost~\cite{NIPS2007_60}, which leverages the correspondence
between importance-weighting and rejection sampling in the large data
setting: the large data set is iteratively subjected to a weak learner
on a manageable subsample.  The scheme herein is akin to a degenerate
two-stage version of FilterBoost, but the differences are important.
Theoretically, the second stage of our procedure is on an unknown
superset of hypotheses. This mandates enforcing a minimum sampling
probability related to the quality of the initial model.  Practically,
in typical workflows there really is a large fixed dataset from which a
single subsample is extracted and subsequently used for (multiple!) model
selection experiments. This commonplace scenario motivates the study of
this particular two-stage procedure.

The architecture of cascading classifiers with increasing computational
complexity to achieve an efficient ensemble was popularized
by Viola and Jones~\cite{Viola01}. More recently, the FCBoost
algorithm~\cite{NIPS2010_0646} was introduced, which implements fully
automatic cascade design within a boosting framework.  While similar
architecturally, cascades focus on short-circuiting a sequential chain of
classifiers in order to minimize evaluation time complexity, whereas
the primary concern here is reduction of training set size. This is
necessitated by the superlinear scaling in training time complexity of
many popular methods such as decision trees and Gaussian processes.
Nonetheless the procedure outlined herein can be considered a simple
two-stage cascade, leading to the same important differences as
highlighted in the previous paragraph.

Active learning is concerned with achieving good generalization
while limiting the number of labels revealed to the learner, and our
results here are clearly related.  In particular our subsampling
rate is lower-bounded similarly to worst-case label complexity
results~\cite{Beygelzimer08}.  It is tempting to conclude that more
sophisticated active learning approaches would achieve lower subsampling
rates, but again the fact that the second stage of our procedure is
on an unknown superset of hypotheses is important.  In particular,
since disagreement regions \cite{hanneke2007bound} only increase on
the hypothesis superset, for our particular scenario an active learning 
algorithm which attempts to exploit what appears to be an isolated empirical 
minimizer is subject to poor worst-case behaviour.

Interesting connections exist between this work and
research directions in sample compression and Minimum Description 
Length (MDL). Sample compression algorithms \cite{floyd1995sample}
learn classifiers that can be described by only a small fraction of the
training data. Algorithms based on the MDL
principle \cite{grunwald2007minimum} treat training examples and models 
as data that need to be transmitted to a receiver and they operate by 
conceptually finding a code that minimizes communication costs.
However, both sample compression and MDL are aware of the subsequent
steps in the protocol they are operating. In sample compression the 
reconstruction function has to match the compression function and in
MDL the code used to communicate has to be known to the receiver. 
In our work, the subsampling step is oblivious to the subsequent 
ERM step which provides great flexibility in practical applications.

Subsampling to mitigate computational constraints during training is
an old idea and a common practice in the machine learning community.
The exact setup considered here was investigated empirically almost 2
decades ago~\cite{Lewis94}, where a simpler computationally inexpensive
hypothesis was used to select importance-weighted training data for a
more expressive computationally intensive hypothesis.  More recent work
from neural language modeling~\cite{Bengio2008} indicates the issue
of controlling worst-case subsample deviations remains open.  In that
work, the authors address the issue by adapting the compressing hypothesis
to match the empirical minimizer (both of which are learned online).
The approach is this paper is much simpler: we enforce a minimum
sampling probability to upper-bound worst-case empirical variance.

\subsection{Contributions}

We define an optimal subsample selection problem given a compressing
hypothesis using empirical Bernstein bounds.  The solution is a 
simple strategy parameterized by the overall subsample budget.  We prove
a bound on deviations between the subsample empirical risk minimizer and the 
true risk minimizer which compares favorably to ERM on the original sample.

We demonstrate the effectiveness by achieving competitive results on a 
large public dataset for which naive subsampling techniques are not an
effective strategy.

\section{Derivation of the Sampling Strategy}

Our starting point is the old practitioner's chestnut: when faced with a
binary classification problem with a highly unbalanced label distribution,
discard examples associated with the more frequent label until the
relative number of examples with each label is about even.  Remaining
examples, associated with the formally more frequent class label, must
be importance weighted to retain an unbiased sample. Curiously, for
logistic regression this can be done analytically by adjusting the bias
weight after training, and in language modeling this results in large
training time speedups without significant degradation of generalization
performance~\cite{Xu_2011}.

Although this practice is widespread and intuitively reasonable, our
goal is a satisfactory theoretical explanation of the approach which
in turn suggests how to improve the technique. Note that since we can
leverage the labels of the entire dataset, our setup does not obviously
correspond to active learning.  We believe a thorough understanding
requires a non-uniform view of the hypothesis space, and in particular,
our results leverage empirical Bernstein bounds.

A key observation is that subsampling the more frequent class exactly
preserves the empirical 0-1 loss of the best constant hypothesis, because
the discarded points have a loss of 0.  In fact, if the only purpose of
the subsample was to transmit the empirical risk of the best constant
hypothesis, all instances associated with the more frequent class could
be discarded.  However, the subsample will be used for empirical risk
minimization.  By definition, the empirical minimizer on the subsample
will have empirical risk on the subsample at least as good as the best
constant hypothesis. However, the deviation between the empirical risk on
the subsample and the true risk might be large. Fortunately, retaining
the instances associated with the more frequent class has the effect of
bounding the worst-case empirical variance of the loss of the subsample
empirical risk minimizer. This, together with an application of empirical
Bernstein bounds, indicates that the deviation between subsample risk
and true risk is small.

\subsection{A Compressing Hypothesis}

One way to generalize frequent label subsampling is to let the set of initial
predictors considered be richer than the constant predictors.
For instance, the class labels might be approximately balanced. Yet,
when conditioned on a single feature, the class label distribution might
be somewhat imbalanced.  In general, we might be able to easily search
at scale over simple hypothesis spaces prior to subsampling for model
selection: can we take advantage of this?

Consider any hypothesis $\tilde h$ which is guaranteed to be in the set of
hypotheses for the final ERM step.  For now, let us assume the subsampling
procedure does not distort the empirical risk of $\tilde h$. Then, we
can upper bound the subsample empirical risk of the subsample empirical
risk minimizer, which in turn will allow us to bound deviations of the
subsample minimizer from the true underlying distribution.

Formally, consider that we have an i.i.d. empirical sample $\mathbf{X}
= (X_1, \ldots, X_n)$ of size $n$.  With a slight abuse of notation
in what follows, we will consider the loss function fixed and we will
not distinguish between a hypothesis $h$ and the induced loss function
$\ell_h$. Therefore, we do not need to distinguish between features
and labels.  Empirical risk minimization on the original sample would
be driven by the risk
\begin{displaymath}
R_{\mathbf{X}} (h) = \frac{1}{n} \sum_{i=1}^n h (X_i).
\end{displaymath}
Given $h$, our subsampling strategy makes conditionally independent decisions
to sample each $X_i$, where $Q_i \in \{ 0, 1 \}$ is a random
variable indicating whether or not $X_i$ is included in the subsample,
and $P_i = \mathbb{E}[Q_i | \mathbf{X}]$ is the sampling probability.
The final ERM step minimizes importance-weighted empirical risk on the resulting subsample:
\begin{displaymath}
R_{\mathbf{Q}, \mathbf{X}} (h) = \frac{1}{n} \sum_{i=1}^n \frac{Q_i}{P_i} h (X_i).
\end{displaymath}
We want to limit the degradation introduced by subsampling, i.e., bound
deviations between $R_{\mathbf{X}} (h)$ and $R_{\mathbf{Q}, \mathbf{X}}
(h)$.  The empirical Bernstein bound \cite{MaurerP09} suggests that
deviations are driven by the subsample empirical variance
\begin{displaymath}
\begin{aligned}
&\mathbb{V}_n (h | \mathbf{Q},\mathbf{X}) \\
&\;\;\;= \frac{1}{n-1} \sum_{i=1}^n \left(\frac{Q_i}{P_i} h (X_i) - \left( \frac{1}{n} \sum_{i=1}^n h (X_i) \right) \right)^2.
\end{aligned}
\end{displaymath}
It turns out the worst case scenario is when $h$ has high loss on examples
where $P_i$ is small.   For any distribution $D$ and any random variable
$Z \in [0, w]$ we have $\mathbb{V}_D[Z] \leq w \mathbb{E}_D[Z]$. For
the subsample empirical distribution, in particular, $\mathbb{V}_n (h |
\mathbf{Q},\mathbf{X}) \leq 1/(\min_i P_i) R_{\mathbf{Q}, \mathbf{X}}
(h)$.  Thus, we can bound the worst-case subsample empirical variance of
$R_{\mathbf{Q}, \mathbf{X}} (h)$ by enforcing a minimum sampling probability $P_{\min}$. 
If we knew that the subsample empirical minimizer $\hat h$ had subsample
empirical risk $R_{\mathbf{Q}, \mathbf{X}} (\hat h)$, we could choose
$P_{\min} = R_{\mathbf{Q}, \mathbf{X}} (\hat h)$. This choice guarantees
that deviations introduced by subsampling are of the same order as
deviations in the original sample.  Unfortunately we cannot choose
$P_{\min}$ in this fashion as it involves circular reasoning.

Instead we can leverage the compressing hypothesis $\tilde h$ to choose
$P_{\min}$.  In particular, if the subsampling procedure does not distort
the empirical risk of $\tilde h$, then $R_{\mathbf{Q}, \mathbf{X}} (\hat
h) \leq R_{\mathbf{Q}, \mathbf{X}} (\tilde h) \approx R_{\mathbf{X}}
(\tilde h)$, so we can set $P_{\min} = R_{\mathbf{X}} (\tilde h)$.
This indicates bounding the deviation of $\tilde h$ between sample and subsample is critical.  
For $\tilde h$ we can use Bennett's inequality to bound the deviation
introduced by subsampling via the variance of the subsampling procedure,
\begin{displaymath}
\begin{aligned}
&\mathbb{V}_{\mathbf{Q}} (\tilde h | \mathbf{X}) \\
&= \mathbb{E}_{\mathbf{Q}}\left[ \frac{1}{n} \sum_{i=1}^n \left(\frac{Q_i}{P_i} \tilde h (X_i) - \mathbb{E}_{\mathbf{Q}}\left[\frac{1}{n} \sum_{i=1}^n \frac{Q_i}{P_i} \tilde h (X_i) \right] \right)^2 \biggl| \mathbf{X} \right] \\
&= \frac{1}{n} \sum_{i=1}^n \left(\frac{1}{P_i} - 1 \right) \tilde h (X_i)^2.
\end{aligned}
\end{displaymath}

The above considerations motivate the following formulation of optimal
subsampling,
\begin{displaymath}
\begin{aligned}
\min_{\{P_i\}} \frac{1}{n} &\sum P_i \\
\mbox{s.t.} \\
\frac{1}{n} \sum_{i=1}^n &\left(\frac{1}{P_i} - 1\right) \tilde h (X_i)^2 \leq V, \\
\forall i: P_i &\geq P_{\min}.
\end{aligned}
\end{displaymath}
The KKT conditions reveal $P_i = \max\{ P_{\min}, \lambda \tilde h
(X_i) \}$, where $\lambda$ depends upon both the variance budget $V$
and the minimum probability $P_{\min}$.  Thus we will be sampling at a
rate proportional to the instantaneous loss of compressing hypothesis,
subject to a minimum sampling rate.

\subsection{The Sampling Strategy}

The above considerations lead to the following.

\begin{definition}[Sampling Strategy]\label{def:subsample}
Fix a sample $\mathbf{X} = (X_1, \ldots, X_n) \in \mathcal{X}^n$,
let $\lambda > 0$, let $P_{min} > 0$, and let $\tilde h : \mathcal{X} \to [ 0, 1 ]$ be any
hypothesis.  The sampling strategy wrt $(\tilde h, \lambda, P_{min})$ is a set of random 
variables $\mathbf{Q} = (Q_1, \ldots, Q_n)$ that defines a subsample of $\mathbf{X}$,
where the $Q_i \in \{ 0, 1 \}$ have conditional independence $Q_i \perp Q_{j
\neq i}, X_{j \neq i} | X_i$ and conditional expectation
\begin{displaymath}
P_i \doteq \mathbb{E}[Q_i | X_i] = \min \left\{1, \max \left\{ P_{min}, \lambda \tilde h (X_i) \right\} \right\}.
\end{displaymath}
\end{definition}
For this strategy we can prove the following.

\begin{theorem}\label{thm:one}
Let $X$ be a random variable with values in set $\mathcal{X}$ with
distribution $D$, let $\mathbf{X} = (X_1, \ldots, X_n) \sim D^n$ be an
i.i.d. empirical sample of size $n$, let $\mathcal{H}$ be a finite set of
hypotheses $h : \mathcal{X} \to [ 0, 1 ]$, let $\tilde h \in \mathcal{H}$
be any hypothesis with empirical mean $R_{\mathbf{X}} (\tilde h)$,
and let $\mathbf{Q} = (Q_1, \ldots, Q_n)$ be a set of random variables
according to the sampling strategy wrt $(\tilde h, \lambda, P_{min})$.
Let $h^* \in \mathcal{H}$ be any hypothesis with minimum true mean, and
let $\hat h \in \mathcal{H}$ be any hypothesis with minimum subsampled
empirical mean $R_{\mathbf{Q}, \mathbf{X}} (h)$.
For $\delta > 0$, $n \geq 2$ we have with probability at least $1 -
3 \delta$ in $\mathbf{Q}$ and $\mathbf{X}$,
\begin{displaymath}
\begin{aligned}
&\mathbb{E}_{D}[\hat h (X)] \leq \mathbb{E}_{D}[h^* (X)] \\
&+ \left(2 + \sqrt{\frac{R_{\mathbf{X}} (\tilde h)}{P_{min}}}\right)  \sqrt{\frac{2 \ln (\vert \mathcal{H} \vert / \delta)}{n}} \\
&+ \left( \sqrt[4]{\frac{R_{\mathbf{X}} (\tilde h) P_{min}}{\lambda}} +  \frac{2}{3}\right) \left( \frac{2 \ln (\vert \mathcal{H} \vert / \delta)}{P_{min} n} \right)^{3/4} \\
&+ 4 \frac{\ln (\vert \mathcal{H} \vert / \delta)}{P_{min} (n - 1)}.
\end{aligned}
\end{displaymath}
\begin{proof} See the appendix.
\end{proof}
\end{theorem}
Analogous results are possible for infinite hypothesis classes whose
complexity can be suitably controlled.

From Theorem \ref{thm:one} it is clear that our scheme cannot subsample
at a rate below the average loss of the compressing hypothesis without
incurring increasing excess risk; this is analogous to a lossless
compression rate threshold.  However if $P_{min} \geq R_{\mathbf{X}}
(\tilde h)$ and $\lambda \geq 1$, then excess risk is $O (1 / \sqrt{n})
+ O (1 / m^{3/4})$, where $n$ is the original data set size and $m$
is a lower bound on the subsampled data set size.

In practice $P_{min}$ and $\lambda$ are chosen according to the subsample
budget, since the expected size of the subsample is upper bounded by $(P_{min}
+ \lambda R_{\mathbf{X}} (\tilde h)) n$.  Unfortunately there are two 
hyperparameters and the analysis presented here does not guide
the choice except for suggesting the constraints 
$P_{min} \geq R_{\mathbf{X}} (\tilde h)$ and $\lambda \geq 1$; this 
is a subject for future investigation.

For binary classification 0-1 loss, using the best constant predictor as
the compressing hypothesis, $P_{min} = R_{\mathbf{X}} (\tilde h)$, and
$\lambda = 1$, the strategy reduces to the familiar ``subsample instances
with the rarer class label in order to make a balanced data set.''

\section{Experiments}

\begin{figure}
\includegraphics[height=2in]{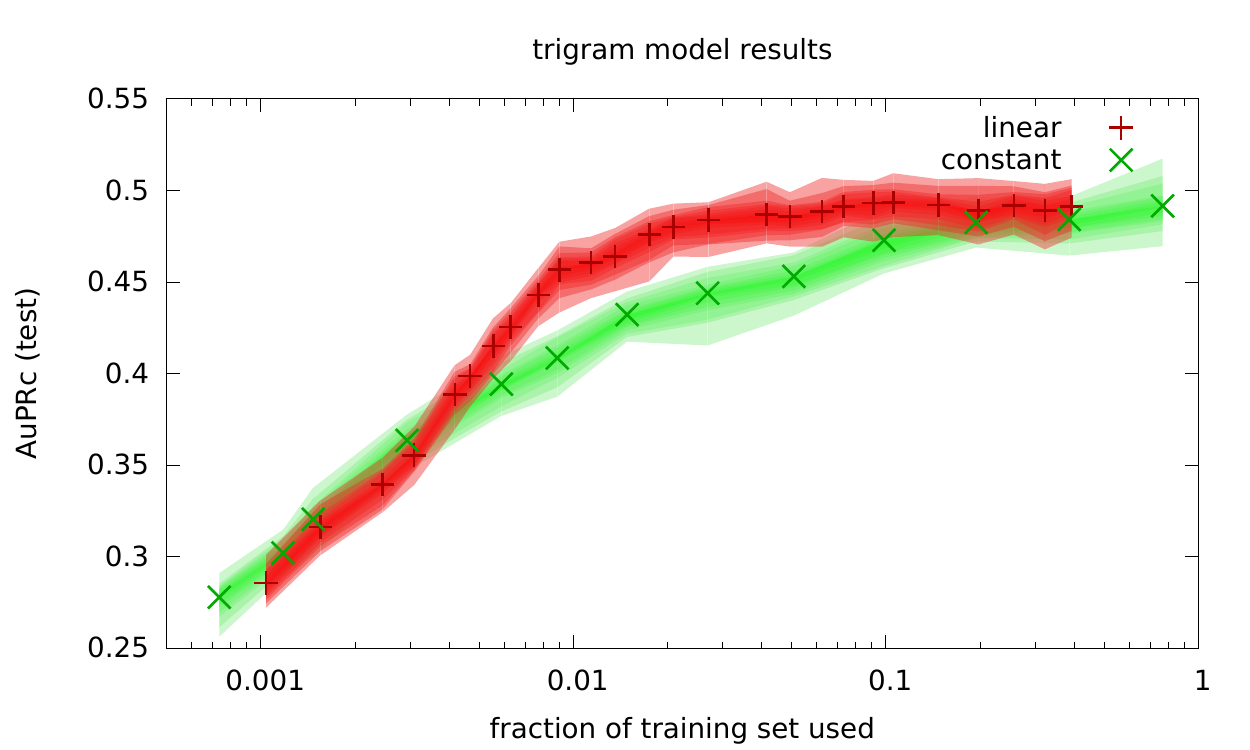}
\vspace{-16pt}
\caption{Test AuPRc for a trigram logistic regression as a function of
subsample fraction and subsample method.  The colored areas around the
data points are 90\% confidence intervals obtained from the bootstrap
distribution of the test set for a fixed predictor.}
\vspace{-16pt}
\label{fig:dnaplot}
\end{figure}

To demonstrate the technique we used the DNA dataset from the 2008 Pascal
Large Scale Learning challenge~\cite{pascallarge}.   This dataset consists
of 50 million instances of 200 base pair oligonucleotides with associated
binary labels corresponding to whether or not the sequence contains
a splice site.  This is a highly imbalanced data set, with 144,823
positives and 49,855,177 negatives.  This dataset is notable because it is a large public data set for
which subsampling has not heretofore been an effective learning
strategy~\cite{SonnenburgFranc2010,Agarwal11}. 

The conventional evaluation metric for this data set is area under
the precision-recall curve.  AuPRCs of circa 0.2 are typical of
``fast'' methods for this dataset, although the best known technique
for this dataset achieves an AuPRc of 0.586 on the validation
set~\cite{SonnenburgFranc2010}.  The labels for the validation set for
this dataset are not published, and are accessible only via a submission
oracle.  We took the original published training set and split it
into training and test sets by reserving the first 1 million instances
as test.   Unless otherwise indicated, we utilize our train/test split
and the reported metrics are not directly comparable with other published
results.  To assess the sensitivity of our results to the exact test set,
we use the bootstrap to estimate the dispersion in the AuPRc. We generate
bootstrap samples of the test set and compute the AuPRc statistic on
each bootstrap sample using the same predictor. In what follows, a 90\%
confidence interval refers to the $5^{\mathrm{th}}$ and $95^{\mathrm{th}}$
quantile of the distribution of AuPRc values obtained this way.

\subsection{Trigram final model}

For our initial (compressing) model we used logistic regression as implemented in
{\tt Vowpal Wabbit}~\cite{vowpalwabbit}, encoding the nucleotide at
each position with a one-hot encoding. This model achieves 0.215 test AuPRc.  

To generate a subsample, we used the initial model to subsample the
original data set as per definition \ref{def:subsample} with $P_{min} =
R_{\mathbf{X}} (\tilde h)$ and for a range of $\lambda$ from 1 to 65536
exponentially spaced; we name this subsampling method \emph{linear}.
For the loss function we used logistic loss, normalized on the training
set to be in the range $[0,1]$.  We compared this to the well-known
and ubiquitously applied strategy of taking all the positively labelled
instances plus a uniform sample of the negatively labelled instances;
we name this subsampling method \emph{constant}.

For our final model we again used logistic regression but included
one-hot encodings of bigrams and trigrams at each position.  Figure
\ref{fig:dnaplot} shows the results of training a trigram model on the
subsample as a function of subsample fraction and subsampling method.
Only the training set is subsampled: the complete test set is used
every time for evaluation.  For the range of subsample fractions roughly
between 1\% and 10\%, the subsample generated via \emph{linear} results
in better test performance; performance at other subsampling rates
is essentially equivalent.  At a 7\% fraction, \emph{linear} achieves
a test AuPRc of 0.491 with 90\% confidence interval $[0.474,0.506]$.
This is equivalent to training a trigram logistic regression on the
entire data set, which achieves test AuPRc of 0.494 with 90\% confidence
interval $[0.475, 0.511]$, as summarized in table \ref{tab:trigramdna}.

\begin{table}
\begin{tabular}{|c|c|c|}
Subsample & Training Set & \multirow{2}{*}{Test AuPRc (90\% CI)} \\ 
Method & Size & \\ \hline
\emph{constant} & 4,829,983 & 0.472 ([0.454, 0.489]) \\ 
\emph{linear} & 3,592,113 & 0.491 ([0.474, 0.506]) \\ 
full data & 49,000,000 & 0.494 ([0.475, 0.511]) 
\end{tabular}
\caption{Performance of trigram logistic regression on the DNA dataset using different subsampling strategies.  \emph{linear} generalizes better than \emph{constant} with less training data, and is nearly equivalent to training on the full data set.}
\vspace{-16pt}
\label{tab:trigramdna}
\end{table}

The confusion matrix for the initial \emph{linear} model on the training set provides
some intuition regarding the improved efficiency.
\begin{center}
\begin{tabular}{l|l|c|c|}
\multicolumn{2}{c}{}&\multicolumn{2}{c}{Prediction}\\
\cline{3-4}
\multicolumn{2}{c|}{}& Positive & Negative\\
\cline{2-4}
\multirow{2}{*}{Truth}& Positive & 4677 & 137252 \\
\cline{2-4}
& Negative & 3412 & 48854660 \\
\cline{2-4}
\end{tabular}
\end{center}
Both \emph{linear} and \emph{constant} will have a positively labeled
instance enriched subsample, the latter by explicit design, and the
former because most true positives have large logistic loss using the
initial model.  The \emph{constant} model, however, will have a uniform
subsample of negatively labeled instances.  By contrast, \emph{linear}
will treat the 3412 false positives similarly to positively labeled
instances, and furthermore negatively labelled instances that are near
the classification boundary will be more likely to be incorporated into
the subsample.  This non-uniform view of the negatively labelled data
helps prevent overfitting in the subsample.

\subsection{GBM final model}

Next we experimented with the \texttt{gbm} decision tree
package~\cite{Ridgeway05}, with which using the complete dataset is not
feasible on a current commodity desktop machine.
%
We used the trigram feature encoding using depth 3 trees, i.e.,
3-way interactions between trigrams.  For the initial model we used
\emph{constant} and \emph{linear} as above, but additionally employed
\emph{trigram} which is the final model from the previous experiment
trained on the entire data set.  The results are in table \ref{tab:gbmdna}.
\begin{table}
\begin{tabular}{|c|c|c|}
Method & Size & Test AuPRc (90\% CI) \\ \hline
\emph{constant} & 873,405 & 0.480 ([0.462, 0.495]) \\ 
\emph{linear} & 840,118 & 0.524 ([0.501, 0.542]) \\ 
\emph{trigram} & 834,131 & 0.567 ([0.545, 0.582]) 
\end{tabular}
\caption{Performance of \texttt{gbm} on the DNA dataset using different subsampling strategies.}
\vspace{-16pt}
\label{tab:gbmdna}
\end{table}

\texttt{gbm} utilizing the subsample defined by \emph{trigram} achieves AuPRc
of 0.567 $([0.545, 0.582])$, which is better than \emph{trigram} model trained on the 
entire dataset.  Hence, a more computationally demanding
model selection step on a subsample can achieve better results than
a simpler model selection step utilizing all the data.  Furthermore,
this is competitive with the best known solutions, despite \texttt{gbm}
only having access to less than 2\% of the data.

The difference in training time between \emph{linear} and \emph{trigram}
is quite modest: roughly 60 vs. 75 minutes for the entire data set on
a single core of a commodity laptop. On the same hardware \texttt{gbm}
takes roughly 3 days to produce a 10,000 tree ensemble using 800,000
examples.  We speculate that for some domains there is a knee in the
performance of classifiers relative to computational effort, such that
reasonable performance can be achieved with modest effort, as with
the \emph{trigram} model above.  In such cases, using a ``sweet spot''
model as the compressing hypothesis for a more computationally demanding
technique is a productive strategy.

\section{Conclusion}

We have derived a general technique for subsampling prior to model
selection which leverages a compressing hypothesis,  proven a deviation
bound for subsampled empirical risk minimization which compares favorably
to empirical risk minimization on the original sample, and demonstrated 
the approach experimentally on a large public dataset. 

These results enable the beneficial use of effective but
non-scalable learning algorithms on larger datasets.

\section{Appendix (Proofs)}

The next two Theorems are from~\cite{MaurerP09}, slightly modified to
range over $[0, w]$.

\begin{theorem}[Bennett's Inequality]
Let $Z, Z_1, \ldots Z_n$ be i.i.d. random variables with values in $[0,
w]$ and let $\delta > 0$.  With probability at least $1 - \delta$ in
the i.i.d. vector $\mathbf{Z} = (Z_1, \ldots Z_n)$ we have
\begin{displaymath}
\mathbb{E}[Z] - \frac{1}{n} \sum_{i=1}^n Z_i \leq \sqrt{\frac{2 \mathbb{V} (\mathbf{Z}) \ln 1 / \delta}{n}} + \frac{w \ln 1 / \delta}{3 n},
\end{displaymath}
where $\mathbb{V} (Z) = \mathbb{E}[(Z - \mathbb{E}[Z])^2]$ is the
variance.
\end{theorem}

\begin{theorem}[Empirical Bernstein
Inequality]\label{thm:empiricalbennett}
Let $Z, Z_1, \ldots Z_n$ be i.i.d. random variables with values in $[0,
w]$ and let $\delta > 0$.  With probability at least $1 - \delta$ in
the i.i.d. vector $\mathbf{Z} = (Z_1, \ldots Z_n)$ we have
\begin{displaymath}
\mathbb{E}[Z] - \frac{1}{n} \sum_{i=1}^n Z_i \leq \sqrt{\frac{2 \mathbb{V}_n (\mathbf{Z}) \ln 2 / \delta}{n}} + \frac{7 w \ln 2 / \delta}{3 (n - 1)},
\end{displaymath}
where $\mathbb{V}_n (\mathbf{Z}) = (1/(n-1)) \sum_{i=1}^n (Z_i - (1/n)
\sum_{j=1}^n Z_j)^2$ is the empirical variance.
\end{theorem}

\begin{lemma}
Fix a sample $\mathbf{X} = (X_1, \ldots, X_n)$, let $\mathcal{H}$ be
a finite set of hypotheses $h : \mathcal{X} \to [ 0, 1 ]$, let $\tilde
h \in \mathcal{H}$ be any hypothesis with empirical mean $R_{\mathbf{X}} (\tilde h)$,
and let $\mathbf{Q} = (Q_1, \ldots, Q_n)$ be a set of random variables
according to the sampling strategy wrt $(\tilde h, \lambda, P_{min})$.  For $\delta >
0$ we have with probability at least $1 - \delta$ in $\mathbf{Q}$,
\begin{displaymath}
\begin{aligned}
\frac{1}{n} &\sum_{i=1}^n \frac{Q_i}{P_i} \tilde h (X_i) \\
&\leq R_{\mathbf{X}} (\tilde h) + \sqrt{\frac{2 R_{\mathbf{X}} (\tilde h) \ln (\vert \mathcal{H} \vert / \delta)}{\lambda n}} + \frac{\ln (\vert \mathcal{H} \vert / \delta)}{3 P_{min} n}.
\end{aligned}
\end{displaymath}
\begin{proof} First we bound the variance due to sampling,
\begin{displaymath}
\begin{aligned}
&\mathbb{V}_{\mathbf{Q}} (\tilde h | \mathbf{X}) \\
&= \mathbb{E}_{\mathbf{Q}}\left[ \left( \frac{1}{n} \sum_{i=1}^n \frac{Q_i}{P_i} \tilde h (X_i) - \frac{1}{n} \sum_{i=1}^n \tilde h (X_i) \right)^2 \biggl| \mathbf{X} \right] \\
&= \frac{1}{n} \sum_{i=1}^n \left(\frac{1}{P_i} - 1 \right) \tilde h (X_i)^2 \\
&\leq \frac{1}{n} \sum_{i=1}^n 1_{\lambda h (X_i) < 1} \left(\frac{1}{\lambda \tilde h (X_i)} - 1 \right) \tilde h (X_i)^2 \\
&\leq \frac{1}{\lambda} R_{\mathbf{X}} (\tilde h).
\end{aligned}
\end{displaymath}
Applying Bennett's inequality using range $[0, 1 / P_{min}]$ yields
the desired result.
\end{proof}
\end{lemma}

\begin{lemma}
Fix a sample $\mathbf{X} = (X_1, \ldots, X_n)$, let $\mathcal{H}$ be
a finite set of hypotheses $h : \mathcal{X} \to [ 0, 1 ]$, let $\tilde
h \in \mathcal{H}$ be any hypothesis with empirical mean $R_{\mathbf{X}} (\tilde h)$,
and let $\mathbf{Q} = (Q_1, \ldots, Q_n)$ be a set of random variables
according to the sampling strategy wrt $(\tilde h, \lambda, P_{min})$.  Let $\hat h \in
\mathcal{H}$ be any hypothesis with minimum subsample empirical mean
$R_{\mathbf{Q}, \mathbf{X}} (\hat h)$.  For $\delta > 0$, $n \geq 2$
we have with probability at least $1 - 2 \delta$ in $\mathbf{Q}$,
\begin{displaymath}
\begin{aligned}
\frac{1}{n} &\sum_{i=1}^n \frac{Q_i}{P_i} \hat h (X_i) \geq \frac{1}{n} \sum_{i=1}^n \hat h (X_i) - \sqrt{\frac{R_{\mathbf{X}} (\tilde h)}{P_{min}}} \sqrt{\frac{2 \ln (\vert \mathcal{H} \vert / \delta)}{n}} \\
& - \sqrt[4]{\frac{R_{\mathbf{X}} (\tilde h) P_{min}}{\lambda}} \left( \frac{2 \ln (\vert \mathcal{H} \vert / \delta)}{P_{min} n} \right)^{3/4} - \frac{10 \ln (\vert \mathcal{H} \vert / \delta)}{3 P_{min} (n - 1)}.
\end{aligned}
\end{displaymath}
\begin{proof} First we bound the empirical subsample variance,
\begin{displaymath}
\begin{aligned}
&\mathbb{V}_n (\hat h | \mathbf{Q}, \mathbf{X}) \\
&=  \frac{1}{n} \sum_{i=1}^n \left(\frac{Q_i}{P_i} \hat h (X_i) - \left( \frac{1}{n} \sum_{i=1}^n \frac{Q_i}{P_i} \hat h (X_i) \right) \right)^2 \\
&\leq \frac{1}{P_{min}} \frac{1}{n} \sum_{i=1}^n \frac{Q_i}{P_i} \hat h (X_i) \\
&\leq \frac{1}{P_{min}} \frac{1}{n} \sum_{i=1}^n \frac{Q_i}{P_i} \tilde h (X_i) \\
&\leq  \frac{1}{P_{min}} \left( R_{\mathbf{X}} (\tilde h) + \sqrt{\frac{2 R_{\mathbf{X}} (\tilde h) \ln (\vert \mathcal{H} \vert / \delta)}{\lambda n}} + \frac{\ln (\vert \mathcal{H} \vert / \delta)}{3 P_{min} n} \right),
\end{aligned}
\end{displaymath}
where the first inequality is due to $P_i \geq P_{min}$, the second
due to optimality of $\hat h$ on the filtered sample, and the third due
to the previous lemma.  Applying empirical Bernstein and the concavity
of square root yields
\begin{displaymath}
\begin{aligned}
&\frac{1}{n} \sum_{i=1}^n \frac{Q_i}{P_i} \hat h (X_i) \\
&\geq \frac{1}{n} \sum_{i=1}^n \hat h (X_i)  - \sqrt{\frac{2 \mathbb{V}_n (\hat h | \mathbf{Q}, \mathbf{X}) \ln (\vert \mathcal{H} \vert / \delta)}{n}}\\
&\;\;\; - \frac{7 \ln (\vert \mathcal{H} \vert / \delta)}{3 P_{min} (n - 1)} \\
&\geq \frac{1}{n} \sum_{i=1}^n \hat h (X_i) - \sqrt{\frac{R_{\mathbf{X}} (\tilde h)}{P_{min}}} \sqrt{\frac{2 \ln (\vert \mathcal{H} \vert / \delta)}{n}} \\
&\;\;\; - \sqrt[4]{\frac{R_{\mathbf{X}} (\tilde h) P_{min}}{\lambda}} \left( \frac{2 \ln (\vert \mathcal{H} \vert / \delta)}{P_{min} n} \right)^{3/4} \\
&\;\;\;- \sqrt{\frac{2}{3}} \frac{\ln (\vert \mathcal{H} \vert / \delta)}{P_{min} n}  - \frac{7 \ln (\vert \mathcal{H} \vert / \delta)}{3 P_{min} (n - 1)}.
\end{aligned}
\end{displaymath}
The desired result follows from upper-bounding constants by $10/3$.
\end{proof}
\end{lemma}

\begin{lemma}
Fix a sample $\mathbf{X} = (X_1, \ldots, X_n)$, let $\mathcal{H}$ be a
finite set of hypotheses $h : \mathcal{X} \to [ 0, 1 ]$, let $\tilde h
\in \mathcal{H}$ be any hypothesis with empirical mean $R_{\mathbf{X}} (\tilde h)$,
and let $\mathbf{Q} = (Q_1, \ldots, Q_n)$ be a set of random variables
according to the sampling strategy wrt $(\tilde h, \lambda, P_{min})$.  Let $h^*
\in \mathcal{H}$ be any hypothesis with minimum true mean.  For $\delta
> 0$, $n \geq 2$ we have with probability at least $1 - 2 \delta$
in $\mathbf{Q}$,
\begin{displaymath}
\begin{aligned}
\frac{1}{n} &\sum_{i=1}^n \frac{Q_i}{P_i} h^* (X_i) \\
&\leq \frac{1}{n} \sum_{i=1}^n h^* (X_i) + \sqrt{\frac{R_{\mathbf{X}} (\tilde h)}{P_{min}}} \sqrt{\frac{2 \ln (\vert \mathcal{H} \vert / \delta)}{n}} \\
&\;\; + \frac{2}{3} \left( \frac{2 \ln (\vert \mathcal{H} \vert / \delta)}{P_{min} n} \right)^{3/4} + \frac{\ln (\vert \mathcal{H} \vert / \delta)}{3 P_{min} n}.
\end{aligned}
\end{displaymath}
\begin{proof} First we bound the variance due to sampling,
\begin{displaymath}
\begin{aligned}
&\mathbb{V}_{\mathbf{Q}} (h^* | \mathbf{X})) \\
&= \mathbb{E}_{\mathbf{Q}}\left[ \left( \frac{1}{n} \sum_{i=1}^n \frac{Q_i}{P_i} h^* (X_i) -  \frac{1}{n} \sum_{i=1}^n h^* (X_i)  \right)^2 \biggl| \mathbf{X} \right] \\
&= \frac{1}{n} \sum_{i=1}^n \left(\frac{1}{P_i} - 1 \right) h^* (X_i)^2 \\
&\leq \frac{1 - P_{min}}{P_{min}} \frac{1}{n} \sum_{i=1}^n h^* (X_i)^2 \\
&\leq \frac{1 - P_{min}}{P_{min}} \frac{1}{n} \sum_{i=1}^n h^* (X_i), \\
\end{aligned}
\end{displaymath}
where the first inequality is due to $P_i \geq P_{min}$ and the
second due to $h^* (X_i) \in [0, 1]$.  The true optimality of $h^*$
and Hoeffding's inequality imply
\begin{displaymath}
\begin{aligned}
\frac{1}{n} \sum_{i=1}^n h^* (X_i) \leq R_{\mathbf{X}} (\tilde h) + 2 \sqrt{\frac{\ln (\vert \mathcal{H} \vert / \delta)}{2 n}},
\end{aligned}
\end{displaymath}
therefore
\begin{displaymath}
\begin{aligned}
\mathbb{V}_{\mathbf{Q}} (h^* | \mathbf{X}) &\leq \frac{1 - P_{min}}{P_{min}} \left( R_{\mathbf{X}} (\tilde h) + 2 \sqrt{\frac{\ln (\vert \mathcal{H} \vert / \delta)}{2 n}} \right).
\end{aligned}
\end{displaymath}
Next applying Bennett's inequality and the concavity of square root yields
\begin{displaymath}
\begin{aligned}
&\frac{1}{n} \sum_{i=1}^n \frac{Q_i}{P_i} h^* (X_i) \\
&\leq \frac{1}{n} \sum_{i=1}^n h^* (X_i) + \sqrt{\frac{2 \mathbb{V}_{\mathbf{Q}} (h^* | \mathbf{X}) \ln (\vert \mathcal{H} \vert / \delta)}{n}} + \frac{\ln (\vert \mathcal{H} \vert / \delta)}{3 P_{min} n} \\
&\leq \frac{1}{n} \sum_{i=1}^n h^* (X_i) + \sqrt{R_{\mathbf{X}} (\tilde h) \frac{1 - P_{min}}{P_{min}}} \sqrt{\frac{2
\ln (\vert \mathcal{H} \vert / \delta)}{n}} \\
&\;\;\; + \sqrt[4]{P_{min} (1 - P_{min})^2} \left( \frac{2 \ln (\vert \mathcal{H} \vert / \delta)}{P_{min} n} \right)^{3/4} + \frac{\ln (\vert \mathcal{H} \vert / \delta)}{3 P_{min} n}.
\end{aligned}
\end{displaymath}
The result follows from $\max_{x \in [0, 1]} \sqrt[4]{x (1 - x)^2} < 2/3$.
\end{proof}
\end{lemma}

\begin{proof}[Proof of Theorem 1] Combining the two previous lemmas with
the empirical filtered optimality of $\hat h$ yields
\begin{displaymath}
\begin{aligned}
&\frac{1}{n} \sum_{i=1}^n \hat h (X_i) - \frac{1}{n} \sum_{i=1}^n h^* (X_i) \\
&\leq 2 \sqrt{\frac{R_{\mathbf{X}} (\tilde h)}{P_{min}}} \sqrt{\frac{2 \ln (\vert \mathcal{H} \vert / \delta)}{n}}  \\
&\;\;\;+ \left( \sqrt[4]{\frac{R_{\mathbf{X}} (\tilde h) P_{min}}{\lambda}} +  \frac{2}{3}\right)   \left( \frac{2 \ln (\vert \mathcal{H} \vert / \delta)}{P_{min} n} \right)^{3/4} \\
&\;\;\;+ 4 \frac{\ln (\vert \mathcal{H} \vert / \delta)}{P_{min} (n - 1)}.
\end{aligned}
\end{displaymath}
Applying Hoeffding's inequality twice yields the desired result.
\end{proof}

\bibliography{subsample}
\bibliographystyle{icml2013}

\end{document}